%% file: gp-grief_tmp.tex
\documentclass[]{article}
\usepackage[accepted]{icml2018}
\usepackage{natbib}
\usepackage[utf8]{inputenc} 
\usepackage[T1]{fontenc}    
\usepackage[hidelinks]{hyperref}       
\usepackage{url}            
\usepackage{booktabs}       
\usepackage{amsfonts}       
\usepackage{nicefrac}       
\usepackage{microtype}      
\usepackage{subcaption}
\usepackage{algorithm}
\usepackage{algorithmic}
\usepackage{amssymb,relsize}
\usepackage[intlimits]{amsmath}
\usepackage{array}
\usepackage{mathabx}
\usepackage{booktabs}
\usepackage{appendix}
\usepackage{placeins}
\usepackage{amsthm}
\usepackage{url}
\usepackage{enumitem}
\usepackage{subfiles}
\usepackage{caption}
\usepackage{graphicx,epstopdf} 
\usepackage{cleveref}

\newcommand{\mat}[1]{\ensuremath \boldsymbol{\mathbf{#1}}} 
\newcommand{\order}[1]{\ensuremath{\mathcal{O}(#1)}}
\newcommand{\inR}[1]{\ensuremath \in \mathbb{R}^{#1}}
\newcommand{\K}{\mat{K}}
\newcommand{\Kuu}{\K_\text{U,U}}
\newcommand{\Kxu}{\K_\text{X,U}}
\newcommand{\Kux}{\K_\text{U,X}}
\newcommand{\Kxx}{\K_\text{X,X}}
\newcommand{\Q}{\mat{Q}}

\newcommand{\alp}{\boldsymbol \alpha}
\newcommand{\I}[1]{\mat{I}_{#1}}
\newcommand{\x}{\mat{x}}
\newcommand{\z}{\mat{z}}
\newcommand{\y}{\mat{y}}
\newcommand{\V}{\mat{V}}

\renewcommand{\S}{\mat{S}}

\newcommand{\nystrom}{Nystr\"{o}m}

\newcommand{\mPhi}{\ensuremath \mat{\Phi}}
  {\algorithm\renewcommand{\thealgorithm}{#1}}
  {\endalgorithm}

\newtheorem{theorem}{Theorem}
\newcommand{\rowKhatriRao}[3]{
	\arraycolsep=2pt
	\def\arraystretch{0.1}
	\begin{align} \label{#3}
	&#2 =  \bigast \limits_{i=1}^d #1^{(i)}\\ \nonumber
	&= \left(
	\begin{array}{ccccccc}
	#1^{(1)}(1,:) &\otimes& #1^{(2)}(1,:) &\otimes& \cdots &\otimes& #1^{(d)}(1,:)\\
	#1^{(1)}(2,:) &\otimes& #1^{(2)}(2,:) &\otimes& \cdots &\otimes& #1^{(d)}(2,:)\\
	\vdots && \vdots && \ddots && \vdots\\
	#1^{(1)}(n,:) &\otimes& #1^{(2)}(n,:) &\otimes& \cdots &\otimes& #1^{(d)}(n,:)
	\end{array} \right),
	\end{align}
}

\newcommand{\lesslines}{\looseness=-1} 

\newcommand{\dobib}{ 
	\bibliography{peripheral,core}
	\bibliographystyle{icml2018}
}

\setlist[itemize]{leftmargin=*}

\icmltitlerunning{Scalable Gaussian Processes with Grid-Structured Eigenfunctions}

\begin{document}
\twocolumn[\icmltitle{Scalable Gaussian Processes with Grid-Structured Eigenfunctions (GP-GRIEF)}


\icmlsetsymbol{equal}{*}

\begin{icmlauthorlist}
	\icmlauthor{Trefor W. Evans}{to}
	\icmlauthor{Prasanth B. Nair}{to}
\end{icmlauthorlist}
\icmlaffiliation{to}{University of Toronto, Canada}

\icmlcorrespondingauthor{Trefor W. Evans}{trefor.evans@mail.utoronto.ca}
\icmlcorrespondingauthor{Prasanth B. Nair}{pbn@utias.utoronto.ca}

\icmlkeywords{Gaussian Processes, GP, Eigenfunctions, Nystrom, Bayesian, Nonparametric, Kronecker Product, Khatri-Rao Product}

\vskip 0.3in
]

\printAffiliationsAndNotice{} 
\begin{abstract}
\lesslines
We introduce a kernel approximation strategy that enables computation of the Gaussian process log marginal likelihood and all hyperparameter derivatives in \order{p} time.
Our GRIEF kernel consists of $p$ eigenfunctions found using a \nystrom\ approximation from a dense Cartesian product grid of inducing points.
By exploiting algebraic properties of Kronecker and Khatri-Rao tensor products, 
computational complexity of the training procedure can be practically \emph{independent} of the number of inducing points. 
This allows us to use arbitrarily many inducing points to achieve a globally accurate kernel approximation, even in high-dimensional problems.
The fast likelihood evaluation enables type-I or II Bayesian inference on large-scale datasets.
We benchmark our algorithms on real-world problems with up to two-million training points and $10^{33}$ inducing points.
\end{abstract}
 
\section{Introduction}
Gaussian process (GP) modelling is a powerful Bayesian approach for classification and regression, however,
it is restricted to modestly sized datasets since training and inference require \order{n^3} time and \order{n^2} storage, where $n$ is the number of training points~\citep{rasmussen_gpml}.
This has motivated the development of approximate GP methods that use a set of $m$ ($\ll n)$ inducing points to reduce time and memory requirements to \order{m^2n+m^3} and \order{mn}, respectively~\cite{smola_sor,snelson_fitc,titsias_vfe,peng_eigengp}.
However, such techniques perform poorly if too few inducing points are used, and
computational savings are lost on complex datasets that require $m$ to be large.

\lesslines
\citet{wilson_kiss} exploited the structure of inducing points placed on a Cartesian product grid, allowing for $m > n$ while dramatically reducing computational demands over an exact GP.
This inducing point structure enables significant performance gains in low-dimensions,
however, time and storage complexities scale \emph{exponentially} with the dataset dimensionality, rendering the technique intractable for general learning problems unless a dimensionality reduction procedure is applied.
In the present work, a Cartesian product grid of inducing points is also considered, however, we show that these computational bottlenecks can be eliminated by identifying and exploiting further structure of the resulting matrices.
The proposed approach leads to a highly scalable algorithm which we call GP-GRIEF~(Gaussian Processes with Grid-Structured Eigenfunctions).
After an initial setup cost of \order{np^2 + dnp + dm^{3/d}}, GP-GRIEF requires only \order{p} time and \order{p} memory per log marginal likelihood evaluation, where 
$d$ denotes the dataset dimensionality, and
$p$ is the number of eigenfunctions that we will describe next.
We emphasize that our complexity is practically \emph{independent} of $m$, which can generally be set arbitrarily high.

GP-GRIEF approximates an exact kernel as a finite sum of eigenfunctions 
which we accurately compute using the \nystrom\ approximation conditioned on a huge number of inducing points.
In other words, our model is sparse in the kernel eigenfunctions rather than the number of inducing points, which can greatly exceed the size of the training set due to the structure we introduce.
This is attractive since it is well-known that eigenfunctions produce the most compact representation among orthogonal basis functions.
Although the eigenfunctions used are approximate, we demonstrate convergence in the limit of large $m$.
Additionally, our ability to fill out the input space with inducing points enables accurate global approximations of the eigenfunctions, even at test locations far from the training data.
These basis functions also live in a reproducing kernel Hilbert space, unlike some other sparse GPs whose bases have a pre-specified form (e.g. \citet{quin_ssgp}).
We summarize our main contributions below
\begin{itemize}
\item We break the \emph{curse of dimensionality} incurred by placing inducing points on a full Cartesian product grid.
Typically, a grid of inducing points results in a computational complexity that scales exponentially in $d$, however, we reduce this complexity to \emph{linear} in $d$ by exploiting algebraic properties of Kronecker and Khatri-Rao products. 
%
\item \lesslines We practically eliminate dependence of the inducing point quantity, $m$, on computational complexity.
This allows us to choose $m \gg n$ to provide a highly accurate kernel approximation, even at locations far from the training data.
\item We show that the \nystrom\ eigenfunction approximation becomes exact for large $m$, which is achievable thanks to the structure and algebra we introduce.
\item Applications of the developed algebra are discussed to enable the extension of structured kernel interpolation methods for high-dimensional problems.
We also develop an efficient preconditioner for general kernel matrices.
%
\item We discuss a flexible parametrization of the GRIEF kernel through a re-weighting of the kernel eigenfunctions.
This admits computation of the log marginal likelihood, along with all $p+1$ hyperparameter derivatives in \order{p}.
\item Finally, we demonstrate type-I Bayesian inference on real-world datasets with up to 2~million training points and $m=10^{33}$ inducing points. 
\end{itemize}
We begin with a review of GPs in \cref{sec:GP}, and we outline an eigenfunction kernel approximation in \cref{sec:kernel}.
\Cref{sec:structure} demonstrates why we should use many inducing points and subsequently develops the algebra necessary to make $m \gg n$ efficient and stable.
\Cref{sec:mercer_expansion} outlines a kernel reparameterization that enables efficient type-I Bayesian inference, even for large datasets.
We finish with numerical studies in \cref{sec:experiments}, demonstrating the performance of GP-GRIEF on real-world datasets.

\section{Background on Gaussian Processes}
\label{sec:background}
\label{sec:GP}
We will employ Gaussian processes (GPs) as non-parametric prior distributions over the latent function which generated the training dataset. 
It is assumed that the dataset is corrupted by independent Gaussian noise with variance $\sigma^2 \geq 0$
and that the latent function is drawn from a Gaussian process with zero mean and covariance determined by the kernel $k : \mathbb{R}^{d} \times \mathbb{R}^{d} \rightarrow \mathbb{R}$.
Considering a regression problem, the log marginal likelihood (LML) of the training targets, $\mat{y} \inR{n}$, can be written as
\begin{multline} \label{eqn:likelihood}
\log \mathcal{P}(\mat{y} | \mat{\theta}, \sigma^2, \mat{X}) = 
-\tfrac{1}{2} \log | \Kxx + \sigma^2 \I{n} |
-\\ \tfrac{1}{2} \mat{y}^T (\Kxx + \sigma^2 \I{n})^{-1} \mat{y} 
-\tfrac{n}{2} \log(2\pi),
\end{multline}
where 
$\mat{X} = \{\mat{x}_i \inR{d}\}_{i=1}^n$ is the set of $n$ training point input locations,
$[\K_\text{A,B}]_{i,j} = k(\mat{a}_i,\mat{b}_j)$ such that $\Kxx \inR{n \times n}$ is the kernel covariance matrix evaluated on the training dataset, and
we assume the kernel is parametrized by the hyperparameters, $\mat{\theta}$.
If we consider type II Bayesian inference, we would like to select the hyperparameters $\{\sigma^2, \mat{\theta}\}$ that maximize the LML.
After hyperparameter estimation, inference can be carried out at an untried point, $\mat{x}_* \inR{d}$, giving the posterior distribution of the prediction $y_* \inR{}$
\begin{multline} \label{eqn:posterior}
y_*| \mat{\theta}, \sigma^2, \mat{X}, \mat{x}_* \sim \mathcal{N}\left(\mathbb{E}[y_*],\ \mathbb{V}[y_*]\right),\\ 
\begin{split}
\mathbb{E}[y_*] &= \mat{K}_{\text{x}_*,\text{X}} (\Kxx + \sigma^2 \I{n})^{-1} \mat{y},\\
\mathbb{V}[y_*] &= \mat{K}_{\text{x}_*,\text{x}_*} - \mat{K}_{\text{x}_*,\text{X}}(\Kxx + \sigma^2 \I{n})^{-1}\mat{K}_{\text{X},\text{x}_*}.
\end{split}
\end{multline}
If we take a fully Bayesian (type I) approach, then we integrate out the hyperparameters by considering the hyperparameter posterior.
This generally results in analytically intractable integrals which require techniques such as Markov-Chain Monte Carlo (MCMC) sampling, or one of its variants~\cite{neal_mcmc_gp}.

\section{Eigenfunction Kernel Approximation}
\label{sec:kernel}
We consider a compact representation of the GP prior using a truncated Mercer expansion of the kernel~$k$.
We use the first $p$ eigenfunctions which we approximate numerically using a \nystrom\ approximation~\cite{peng_eigengp}
\begin{equation}
\begin{split} \label{eqn:nystrom_kernel}
\widetilde{k}(\x,\z) 
&= \sum_{i=1}^p 
			   \big( \underbrace{\lambda_i^{{-}\frac{1}{2}} \K_{\x,\text{U}} \mat{q}_i}_{\phi_i(\x)} \big)
               \big( \underbrace{\lambda_i^{{-}\frac{1}{2}} \K_{\z,\text{U}} \mat{q}_i}_{\phi_i(\z)} \big)\\
&= \K_{\x,\text{U}} \mat{Q} \S_p^T \mat{\Lambda}_p^{-1}\S_p\mat{Q}^T \K_{\text{U},\z}
\approx k(\x,\z),
\end{split}
\end{equation}
where
$\text{U} = \{\mat{u}_i \inR{d}\}_{i=1}^m$ refers to the set of $m$ inducing point locations;
$\mat{\Lambda}, \Q \inR{m\times m}$ are diagonal and unitary matrices containing the eigenvalues and eigenvectors of $\K_{\text{U},\text{U}}$, respectively;
$\lambda_i$ and $\mat{q}_i$ denote the $i$th largest eigenvalue and corresponding eigenvector of $\K_{\text{U},\text{U}}$, respectively; 
$\S_p \inR{p \times m}$ is a sparse selection matrix where $\S_p(i,:)$ contains one value set to unity in the column corresponding to the index of the $i$th largest value on the diagonal of $\mat{\Lambda}$; and
we use the shorthand notation $\mat{\Lambda}_p = \S_p \mat{\Lambda} \S_p^T = \text{diag} (\mat{\lambda}_p) \inR{p\times p}$ to denote a diagonal matrix containing the $p$ largest eigenvalues of $\K_\text{U,U}$, sorted in descending order.
$\phi_i(\mat{x})$ is the numerical approximation of the $i$th eigenfunction evaluated at the input $\mat{x}$, scaled by the root of the $i$th eigenvalue.
We only explicitly compute this scaled eigenfunction for numerical stability, as we will discuss later.
Using the kernel $\widetilde{k}$, the prior covariance matrix on the training set becomes
\begin{equation} \label{eqn:cov}
\widetilde{\K}_\text{X,X} = 
\underbrace{\K_{\text{X},\text{U}} \mat{Q} \S_p^T \mat{\Lambda}_p^{{-}\frac{1}{2}}}_{\mat{\Phi}} 
\underbrace{\mat{\Lambda}_p^{{-}\frac{1}{2}} \S_p\mat{Q}^T \K_{\text{U},\text{X}}}_{\mat{\Phi}^T},
\end{equation}
where the columns of $\mat{\Phi} \inR{n \times p}$ contain the $p$ scaled eigenfunctions of our kernel evaluated on the training set.
Observe that if U is randomly sampled from X, then $\widetilde{\K}_\text{X,X}$ is the same covariance matrix from the ``\nystrom\ method'' of \citet{williams_nystrom},
however, since we have replaced the kernel and not just the covariance matrix, we recover a valid probabilistic model~\cite{peng_eigengp}.

While $\widetilde{\K}_\text{X,X}$ has a rank of at most $p$, \citet{peng_eigengp} show how a correction can be added to $\widetilde{k}$ (\cref{eqn:nystrom_kernel}) to give a full rank covariance matrix (provided $k$ does also).
The resulting GP will be non-degenerate.
We can write this correction as
$
\delta(\x-\z)(k(\x,\z) - \widetilde{k}(\x,\z)),
$
where $\delta(a) = 1$ if $a=0$, else 0. 
This correction term does not affect the computational complexity of GP training, however, we find it does not generally improve performance over the unmodified $\widetilde{k}$.
We do not consider this correction in further discussion.

\section{Grid-Structured Eigenfunctions (GRIEF)}
\label{sec:structure}
Previous work employing \nystrom\ approximations in kernel methods require $m$ to be small (often $\ll n$) to yield computational benefits.
As a result, the choice of inducing point locations, $\text{U}$, has a great influence on the approximation accuracy, and many techniques have been proposed to choose U effectively~\cite{smola_greedy_nystrom, drineas_nystrom, zhang_nystrom, belabbas_nystrom, kumar_nystrom_sampling, wang_nystrom, gittens_nystrom, li_nystrom, musco_leverage_nystrom}. 
In this work, we would instead like to use \emph{so many} inducing points that carefully optimizing the distribution of $\text{U}$ is unnecessary. 
We will even consider $m \gg n$.
The following result shows how an eigenfunction approximation can be improved by using many inducing points.
\begin{theorem} \label{thm:nystrom_m_converg}
If the $i$th eigenvalue of $k$ is simple and non-zero and $\text{U} \supset \text{X}$,
a \nystrom\ approximation of the $i$th kernel eigenfunction converges in the limit of large $m$,
\begin{equation}
\mat{q}_i^{(n)} = \lim\limits_{m \rightarrow \infty} \sqrt{\frac{m}{n}} \frac{1}{\lambda_i^{(m)}} \mat{K}_\text{X,U} \mat{q}^{(m)}_i,
\end{equation}
where 
$\lambda_i^{(m)} \inR{}$ and $\mat{q}^{(m)}_i \inR{m}$ are the $i$th largest eigenvalue and corresponding eigenvector of $\mat{K}_\text{U,U}$, respectively. 
$\mat{q}^{(n)}_i$ is the kernel eigenfunction corresponding to the $i$th largest eigenvalue, evaluated on the set X.
\end{theorem}
\begin{proof}
We begin by constructing a \nystrom\ approximation of the eigenfunction evaluated on U, using X as inducing points.
From theorem 3.5 of \cite{baker_nystrom}, as $m \rightarrow \infty$, 
\begin{equation}
\mat{q}_i^{(m)} = \sqrt{\frac{n}{m}} \frac{1}{\lambda_i^{(n)}} \mat{K}_\text{U,X} \mat{q}^{(n)}_i,
\end{equation}
where we assume that the $i$th eigenvalue of $k$ is simple and non-zero.
Multiplying both sides by $\Kxu \Kuu^{-1}$,
\begin{align}
\Kxu \Kuu^{-1} \mat{q}_i^{(m)} &= \sqrt{\frac{n}{m}} \frac{1}{\lambda_i^{(n)}} \Kxu \Kuu^{-1}\mat{K}_\text{U,X} \mat{q}^{(n)}_i.
\end{align}
Since $\text{U} \supset \text{X}$, $\Kxu = \S_n \Kuu$ is a subset of the rows of $\Kuu$, where $\S_n \inR{n \times m}$ is a selection matrix.
We can write
$\Kxu \Kuu^{-1} \Kux = \S_n \Kuu \Kuu^{-1} \Kuu\S_n^T = \S_n \Kuu\S_n^T = \Kxx$. 
Additionally, since the eigenvector $\mat{q}_i^{(m)}$ satisfies,
$\Kuu^{-1} \mat{q}_i^{(m)} = \frac{1}{\lambda_i^{(m)}}\mat{q}_i^{(m)}$, we get
\begin{align}
\frac{1}{\lambda_i^{(m)}}\Kxu \mat{q}_i^{(m)} &= \sqrt{\frac{n}{m}} \frac{1}{\lambda_i^{(n)}} \Kxx \mat{q}^{(n)}_i.
\end{align}
Noting that
$\Kxx \mat{q}^{(n)}_i = \lambda_i^{(n)} \mat{q}^{(n)}_i$
completes the proof.
\end{proof}

\lesslines
For multiple eigenvalues, it can similarly be shown that the $i$th approximated eigenfunction converges to lie within the linear space of eigenfunctions corresponding to the $i$th eigenvalue of $k$ as $m\rightarrow \infty$.

We can use a large $m$ by distributing inducing points on a Cartesian tensor product grid%
\footnote{
\lesslines
We want $\text{U}$ to be sampled from the same distribution as the training data. 
Approximating the data distribution by placing $\text{U}$ on a grid is easy to do by various means as a quick preprocessing step.
}.
\citet{saatci_phd} demonstrated efficient GP inference when training points are distributed in this way by exploiting Kronecker matrix algebra.
We will assume this grid structure for our inducing points, i.e. U will form a grid.
If the covariance kernel satisfies the product correlation rule (as many popular multidimensional kernels do), i.e.
$k(\mat{x},\mat{z}) = \prod_{i=1}^d k_i(x_i,z_i)$, then $\Kuu \inR{m \times m}$ inherits the Kronecker product form $\K_\text{U,U} {=} \bigotimes_{i=1}^d \K_\text{U,U}^{(i)}$, where $\otimes$ is the Kronecker product~\cite{van_loan_kron}.
$\K_\text{U,U}^{(i)} \inR{\widebar{m} \times \widebar{m}}$ are one-dimensional kernel covariance matrices for a 
slice of the input space grid along the $i$th dimension, and 
$\widebar{m} {=} \sqrt[d]{m} {\approx}$\order{10} is the number of inducing points we choose along each dimension of the full grid.
It is evident that the Kronecker product leads to a large, expansed matrix from smaller ones,
therefore, it is very advantageous to manipulate and store these small matrices without ``unpacking'' them, or explicitly computing the Kronecker product.
Exploiting this structure decreases the storage of $\K_\text{U,U}$ from  
$\order{m^2} \rightarrow \order{dm^{2/d}}{=}\order{d\widebar{m}^{2}}$,
and the cost of matrix-vector products with $\Kuu$ from 
$\order{m^2} {\rightarrow} \order{dm^{(d{+}1)/d}}{=}\order{d\widebar{m}^{d+1}}$.
Additionally, the cost of the eigen-decomposition of  $\K_\text{U,U}$ decreases from 
$\order{m^3} \rightarrow \order{dm^{3/d}}{=}\order{d\widebar{m}^{3}}$, and the eigenvector matrix
$\Q = \bigotimes_{i=1}^d \Q^{(i)}$ and eigenvalue matrix
$\mat{\Lambda} = \text{diag}\big(\bigotimes_{i=1}^d \mat{\lambda}^{(i)}\big)$ both inherit a Kronecker product structure, enabling matrix-vector products with $\widetilde{\K}_\text{X,X}$ in \order{d\widebar{m}^{d+1}} operations~\cite{van_loan_kron, saatci_phd}.

In low-dimensions, exploiting the Kronecker product structure of $\K_\text{U,U} {=} \bigotimes_{i=1}^d \K_\text{U,U}^{(i)}$ can be greatly advantageous, however, we can immediately see from the above complexities that the cost of matrix-vector products%
\footnote{We assume that a conjugate gradient method would be employed for GP training requiring matrix-vector products. Alternative formulations would require columns of $\mat{Q}= \bigotimes_{i=1}^d \Q^{(i)}$ to be expanded which similarly scales exponentially~(\order{\widebar{m}^d}).}
with $\widetilde{\K}_\text{X,X}$ increases \emph{exponentially} in $d$.
The storage requirements will similarly increase exponentially since a vector of length $m{=}\widebar{m}^d$ needs to be stored when a matrix-vector product is made with 
$\Q = \bigotimes_{i=1}^d \Q^{(i)}$, and 
$\K_\text{X,U}$ requires \order{\widebar{m}^dn} storage. 
This poor scaling poses a serious impediment to the application of this approach to high-dimensional datasets.

We now show how to massively decrease time and storage requirements from exponential to \emph{linear} in $d$ by identifying further matrix structure in our problem.

We begin by identifying structure in the exact cross-covariance between train (or test) points and inducing points.
These matrices, e.g. $\K_\text{X,U}$, admit a row-partitioned Khatri-Rao product structure as follows~\cite{nickson_blitzkriging} 
\rowKhatriRao{\Kxu}{\Kxu}{eqn:cross_cov_kr}%
where $\ast$ is the Khatri-Rao product whose computation gives a block Kronecker product matrix~\cite{liu_kron}.
We will always mention how Khatri-Rao product blocks are partitioned.
Since $\K_\text{X,U}^{(i)}$ are only of size $n \times \widebar{m}$, the storage of $\K_\text{X,U}$ has decreased from exponential to linear in $d$: $\order{\widebar{m}^dn} \rightarrow \order{dn\widebar{m}}{\approx}\order{dn}$.
We also observe that the selection matrix $\S_p= \bigast_{i=1}^d \S^{(i)}_p$ can be written as a row-partitioned Khatri-Rao product matrix where each sub-matrix contains one non-zero per row.
Further, by exploiting both Kronecker and Khatri-Rao matrix algebra, our main result below shows that
$\K_{\text{X},\text{U}} \mat{Q} \S_p^T$ can be computed in \order{dnp} time.
This is a substantial reduction over the naive cost of \order{\widebar{m}^dnp} time.
\begin{theorem}
\label{thm:KRrowcol}
The product of a 
row-partitioned Khatri-Rao matrix $\K_\text{X,U}=\bigast_{i=1}^d \K_\text{X,U}^{(i)} \inR{n \times \widebar{m}^d}$, a
Kronecker product matrix $\Q=\bigotimes_{i=1}^d \Q^{(i)} \inR{\widebar{m}^d \times \widebar{m}^d}$, and a
column-partitioned Khatri-Rao matrix $\S_p^T = \bigast_{i=1}^d \big(\S^{(i)}_p\big)^T \inR{\widebar{m}^d \times p}$ can be computed as follows
\begin{align}\label{eqn:KRrowcol}
\K_\text{X,U} \Q \S_p^T =  
\bigodot \limits_{i=1}^d \K_\text{X,U}^{(i)} \Q^{(i)} \big(\S^{(i)}_p\big)^T,
\end{align}
where $\odot$ is the (element-wise) Hadamard product.
This computation only requires products of the smaller matrices
$\K_\text{X,U}^{(i)} \inR{n \times \widebar{m}}$,
$\Q^{(i)} \inR{\widebar{m} \times \widebar{m}}$ and 
$\S^{(i)}_p \inR{p \times \widebar{m}}$.
\end{theorem}
\begin{proof}
First, 
observe that $\K_\text{X,U}\Q = \bigast_{i=1}^d \K^{(i)}_\text{X,U} \Q^{(i)} = \bigast_{i=1}^d \mat{R}^{(i)}$ is a row-partitioned Khatri-Rao product matrix using theorem 2 of \cite{liu_kron}.
Now we must compute a matrix product of row-~and~column-partitioned Khatri-Rao matrices.
We observe that each element of this matrix product is an inner product between two Kronecker product vectors, i.e.
$\big[\K_\text{X,U} \Q \S_p^T]_{ij}
= \big(\bigotimes_{l=1}^d \mat{R}^{(l)}(i,:) \big) \big(\bigotimes_{l=1}^d \mat{S}_p^{T\,(l)}(:,j) \big) 
= \prod_{l=1}^d \mat{R}^{(l)}(i,:)\mat{S}_p^{T\,(l)}(:,j)$.
Writing this in matrix form completes the proof.
\end{proof}
If all 
the sub-matrices were dense, \order{dn\widebar{m}p} time would be required to compute $\K_\text{X,U} \Q \S_p^T$, however,
since $\S_p$ is a sparse selection matrix with one non-zero per row, computation requires just \order{dn \max(p,\widebar{m}c)} ${\approx}$ \order{dnp} time.
We achieve this time by computing only the necessary columns of $\K_\text{X,U}^{(i)} \Q^{(i)}$ and avoiding redundant computations.
The constant $c$ is the average number of non-zeros columns in each of $\{\S_p^{(i)}\}_{i=1}^d$ which is typically \order{1}, however, may be $\widebar{m}$ in the worst case.
Evidently the time complexity is effectively \emph{independent} of the number of inducing points.
Even in the rare worst case where $c=\widebar{m}$ and $\widebar{m}^2 > p$, the scaling is extremely weak; \order{dnm^{{2}/{d}}}.

What results is a kernel composed of basis eigenfunctions that are accurately approximated on a grid of inducing points using a \nystrom\ approximation.
Although $m$ increases exponentially in $d$ given this inducing point structure,
the cost of GP training and inference is not affected. 
We call the resulting model GP-GRIEF (GP with GRId-structured EigenFunctions).
Completing the computations required for GP training and inference require straightforward application of the matrix-determinant and matrix-inversion lemmas which we demonstrate later in \cref{eqn:inv_det_lemmas}.

\paragraph{Eigenvalue Search}
\label{sec:eig_search}
What has not been addressed is how to form $\S_p= \bigast_{i=1}^d \S^{(i)}_p$ and compute $\mat{\lambda}_p$ efficiently.
This requires finding the index locations and values of the largest $p$ eigenvalues in a vector of length $m = \widebar{m}^d$.
In high-dimensions, this task is daunting considering $m$ can easily exceed the number of atoms in the observable universe. 
Fortunately, the resulting vector of eigenvalues, $\text{diag}(\mat{\Lambda}) = \bigotimes_{i=1}^d \mat{\lambda}^{(i)}$, has a Kronecker product structure which we can exploit to develop a fast search algorithm that requires only \order{d\widebar{m}p} time.
To do this, we compute a truncated Kronecker product expansion by keeping only the $p$ largest values after each sequential Kronecker product such that only Kronecker products between length $p$ and length $\widebar{m}$ vectors are computed.
\Cref{alg:eig_search} outlines a more numerically stable version of this search strategy that computes the log of the eigenvalues and also demonstrates how $\mat{S}_p$ is computed.

\newcommand{\logeigs}{\ensuremath{\log \hspace{-0.6mm}\mat{\lambda}_p}}
\newcommand{\diag}{\ensuremath{\text{diag}}}
\newcommand{\idxs}{\ensuremath{\text{idxs}}}
\newcommand{\ord}{\ensuremath{\text{ord}}}
\begin{algorithm}
	\caption{
		Computes $\mat{S}_p$, and $\logeigs$ (the log of the $p$ largest eigenvalues of $\Kuu$).
		We use zero-based array indexing,
		$\text{mod}(a,b)$ computes $a$ mod $b$,
		$|\mat{a}|$ computes the length of $\mat{a}$,
		$\text{sort}_b(\mat{a})$ returns the $\min (|\mat{a}|,b)$ largest elements of $\mat{a}$ in descending order, as well as the indices of these elements in $\mat{a}$, and
		$\lfloor \mat{a} \rfloor$ computes the floor of the elements in $\mat{a}$.
	}
	\label{alg:eig_search}
	\begin{algorithmic}
		\STATE {\bfseries Input:} $\{\mat{\lambda}^{(i)} \inR{\widebar{m}}\}_{i=1}^d$
		\STATE {\bfseries Output:} $\{\mat{S}_p^{(i)}\inR{p \times \widebar{m}}\}_{i=1}^d \quad \& \quad \logeigs \inR{p}$
		\STATE $\logeigs,\ \idxs = \text{sort}_p\big(\log(\mat{\lambda}^{(1)})\big)$
		\FOR{$i=2$ {\bfseries to} $d$} 
		\STATE \mbox{$\logeigs,\ord = \text{sort}_p\big(\logeigs {\otimes} \mat{1}_{\widebar{m}} {+} \mat{1}_{|\logeigs|} \otimes  \log(\mat{\lambda}^{(i)})\big)$}
		\STATE $\idxs = \Big[\begin{array}{cc} \idxs\big(\lfloor \ord / \widebar{m} \rfloor,: \big), & \text{mod}(\ord, \widebar{m})\end{array} \Big]$
		\ENDFOR
		\STATE $\big\{\mat{S}_p^{(i)} = \I{\widebar{m}}\big(\idxs(:,i{-}1),:\big)\big\}_{i=1}^d$
	\end{algorithmic}
\end{algorithm}

\paragraph{Computation in High Dimensions}
Direct use of \cref{thm:KRrowcol} may lead to finite-precision rounding inaccuracies and overflow errors in high dimensions because of the Hadamard product over $d$ matrices.
We can write a more numerically stable version of this algorithm by taking the log of \cref{eqn:KRrowcol}, allowing us to write the computation as a sum of $d$ matrices, rather than a product
\begin{align*}
\K_\text{X,U} \Q \S_p^T =
\bigodot \limits_{i=1}^d \text{sign}( \mat{B}^{(i)} )
\odot
\exp \bigg(\sum \limits_{i=1}^d \log(  \text{abs}\ \mat{B}^{(i)})\bigg),
\end{align*}
where $\mat{B}^{(i)} = \K_\text{X,U}^{(i)} \Q^{(i)} \big(\S^{(i)}_p\big)^T$, and exp, log are computed element-wise.
While the sign matrix is the Hadamard product of $d$ matrices, it contains only $\{-1,0,1\}$ so it is not susceptible to numerical issues.
Also, when the sign of an element is zero, we do not compute the log.
Unfortunately, the exp computation can still lead to numerical issues, however, $\mat{\Phi}$ suffers less because of the rescaling provided by the eigenvalues (i.e. elements of $\mat{\Phi}$ are the quotient of possibly very large or small values).
Since all eigenvalues are positive, we can stably compute $\mat{\Phi}$ as follows
\begin{align*}
&\mat{\Phi} 
= \big(\K_\text{X,U} \Q \S_p^T\big) \mat{\Lambda}_p^{-\frac{1}{2}}
=\\
&\bigodot \limits_{i=1}^d \text{sign}( \mat{B}^{(i)} )
\odot
\exp \bigg(
\sum \limits_{i=1}^d \log(  \text{abs}\ \mat{B}^{(i)})
-
\frac{1}{2}\mat{1}_n\,\logeigs^T \bigg),
\end{align*}
where \logeigs\ is computed by \cref{alg:eig_search}.

\subsection{Preconditioning Applications}
\label{sec:precon}
As an aside remark, we discuss the application of $(\widetilde{\K}_\text{X,X}+\sigma^2 \I{n})^{-1}$ as a preconditioner for the exact kernel matrix $\Kxx + \sigma^2 \I{n}$ in moderately sized problems where \order{n^2} storage is not prohibitive.
The use of $\widetilde{\K}_\text{X,X}$ for matrix preconditioning was explored with notable empirical success by \citet{cutajar_preconditioning} where a sub-set of training data was used as inducing points giving $\text{U} \subset \text{X}$ and $m<n$.
By \cref{thm:nystrom_m_converg}, we know that the \nystrom\ approximation converges for large $m$, and we have shown that we can accommodate $m \gg n$ to provide an accurate low-rank kernel matrix approximation.

\subsection{SKI Applications}
\label{sec:ski}
As a further aside, we discuss how the developed algebra can be applied in a general kernel interpolation setting.
\citet{wilson_kiss} introduced a kernel interpolation perspective to unify inducing point methods wherein the kernel is interpolated from a covariance matrix on the inducing points.
For instance, the subset of regressors (SoR) method \cite{silverman_sor,quinonero_sparse_gpm} can be viewed as a zero-mean GP interpolant of the true kernel while 
\citet{wilson_kiss} proposed a sparse approximate interpolant.
We can denote the interpolated covariance matrix as $\mat{E}\Kuu \mat{E}^T$, where $\mat{E} \inR{n\times m}$ is the interpolation matrix.

In structured kernel interpolation (SKI) the inducing points form a grid such that $\Kuu$ inherits a Kronecker product form.
This can provide dramatic computational advantages, however, SKI suffers from the exponential scaling discussed earlier and so is recommended only for very low-dimensional problems, $d \leq 5$ \cite{wilson_deep}.
We observe that in the case of the GP interpolant (e.g. SoR), as well as the sparse 
interpolant suggested by \citet{wilson_kiss}, the interpolation matrix $\mat{E}$ inherits a row-partitioned Khatri-Rao structure.
This enables direct use of \cref{thm:KRrowcol} to reduce the exponential scaling in $d$ to a linear scaling, and allows SKI to scale to high-dimensional problems.
However, time complexity would scale quadratically in $n$, unlike the proposed GRIEF methods.

\section{Re-weighted Eigenfunction Kernel}
\label{sec:mercer_expansion}
We can approximately recover a wide class of kernels by modifying the weights associated with the kernel eigenfunctions~\cite{buhmann_rbf}.
Here we consider this flexible kernel parametrization for the GRIEF kernel.
Extending \cref{eqn:nystrom_kernel}, we can write the re-weighted GRIEF kernel as
\begin{align} \label{eqn:reweighted_kernel}
\widetilde{k}(\x,\z) 
&= \sum_{i=1}^p 
w_i \phi_i(\x) \phi_i(\z),
\end{align}
where
$\mat{W} \inR{p \times p} = \text{diag}(\mat{w})$,
$\mat{w} = \{w_i > 0\}_{i=1}^p$ are the eigenfunction weights.
The covariance matrix then becomes
\begin{equation} \label{eqn:reweighted_cov}
\widetilde{\K}_\text{X,X} = 
\mat{\Phi} 
\mat{W}
\mat{\Phi}^T,
\end{equation}
and the full set of hyperparameters is $\{\sigma^2, \mat{w}, \mat{\theta}\}$, however, we choose to fix $\mat{\theta}$ so that $\mat{\Phi}$ remains constant.
We can then compute the log marginal likelihood (LML) and all $p+1$ derivatives with respect to $\{\sigma^2,\mat{w}\}$ in \order{p} which is \emph{independent} of $n$.
To do this, we first assume that
$\mat{y}^T\mat{y} \inR{}$,
$\mat{\Phi}^T\mat{y} = \mat{r} \inR{p}$, and
$\mat{\Phi}^T\mat{\Phi}=\mat{A} \inR{p\times p}$
are precomputed, which requires \order{np^2+dnp+d\widebar{m}^3} time, however, this step only needs to be done once before LML iterations begin.
Then, to compute the LML (\cref{eqn:likelihood}), we use the matrix inversion and determinant lemmas to give
\begin{align} \label{eqn:inv_det_lemmas}
&\mat{y}^T \big(\widetilde{\mat{K}}_\text{X,X} + \sigma^2\mat{I}_n\big)^{-1} \mat{y} = 
\sigma^{-2}\big(
\mat{y}^T\mat{y} - \mat{r}^T\mat{P}^{-1}\mat{r}
\big),
\\ \nonumber
&\log \big|\widetilde{\mat{K}}_\text{X,X} + \sigma^2\mat{I}_n\big| =
\log \big|\mat{P}\big| +
\sum \limits_{i=1}^p\log w_i +
(n{-}p)\log \sigma^2,
\end{align}
where $\mat{P} \inR{p \times p} = \sigma^2 \mat{W}^{{-}1} \hspace{-1mm}+ \mat{A}$.
Using these relations, the LML can be computed within \order{p^3} time.
The LML derivatives with respect to all hyperparameters can also be computed in \order{p^3} as shown in the following expression which is derived in \cref{sec:typeI_derivatives} of the supplement
\begin{align*}
&\hspace{-1.7mm}\frac{\partial \text{LML}}{\partial \mat{w}}
{=} \frac{\big(\mat{r} - \mat{A}\mat{P}^{-1}\mat{r}\big)^2}{2\sigma^{4}}
- \frac{\text{diag}\big(\mat{A}\big) - \big(\mat{A} \odot \mat{P}^{-1}\mat{A}\big)^T \mat{1}_p}{2\sigma^{2}},\\
&\hspace{-1.7mm}\frac{\partial \text{LML}}{\partial \sigma^2} 
{=} \frac{\y^T\y {-} 2\mat{r}^T\mat{P}^{-1}\mat{r} {+} \mat{r}^T\mat{P}^{-1}\mat{A}\mat{P}^{-1}\mat{r}}{2\sigma^{4}}
{-} \frac{n {-} \text{Tr}\big(\mat{P}^{-1} \mat{A}\big)}{2\sigma^{2}}.
\end{align*}
To further reduce the per-iteration computational complexity from $\order{p^3} \rightarrow \order{p}$
we apply a linear transformation to the basis functions to make them mutually orthogonal when evaluated on the training data.
We can write the $i$th transformed basis function as
$
\widetilde{\phi}_i(\x) = \sum_{j=1}^p \widetilde{v}_{ji} \widetilde{\Sigma}_{ii}^{-1} \phi_i(\x),
$
where 
$\widetilde{\mat{\Sigma}} \in \mathbb{R}^{\widetilde{p} \times \widetilde{p}}$ is a diagonal matrix containing the non-zero singular values of $\mPhi$,
$\widetilde{\V} \in \mathbb{R}^{p \times \widetilde{p}}$ contains the corresponding right-singular vectors of $\mPhi$, and 
$\widetilde{p} \leq \min(p, n)$.
Using these transformed basis functions, both $\mat{A} = \mat{I}_{\widetilde{p}}$ and $\mat{P}= \sigma^2 \mat{W}^{{-}1} \hspace{-1mm}+ \mat{A}$ become diagonal matrices, enabling evaluation of the LML and all derivatives in \order{p} using \cref{eqn:inv_det_lemmas} and the derivative expressions above.
The transformation requires the singular-value decomposition of $\mPhi$ before LML iterations, however, this precomputation is no more expensive than those discussed previously at \order{n\widetilde{p}^2}.

Since the kernel is now heavily parametrized, maximizing the LML for type-II Bayesian inference is susceptible to overfitting.
Instead, we may choose to take a fully Bayesian type-I approach and integrate out the hyperparameters using hybrid MCMC sampling.
This type-I approach requires far more LML evaluations then type-II (typically $\order{10^5}$), however, the fast \order{p} (or \order{p^3}) evaluations make this tractable even for very large problems
since the cost per iteration is \emph{independent} of the number of training points.
In \cref{sec:typeI_extensions} of the supplement we discuss further extensions to this type-I inference procedure.

\section{Experimental Studies}
\label{sec:experiments}
\paragraph{Two-Dimensional Visualization}
\lesslines
\Cref{fig:2d} shows a comparison between GP-GRIEF and the 
Variational Free Energy (VFE) inducing point approximation~\cite{titsias_vfe} on a two-dimensional test problem with $n=10$ training points generated by the function $f(\mat{x}) {=} \sin(x_1)\sin(x_2)$ and corrupted with $\mathcal{N}(0,0.1)$ noise.
For both models, we use a squared-exponential base kernel, and we estimate the kernel lengthscale and noise variance, $\sigma^2$, by maximizing the log marginal likelihood.
VFE can also select its inducing point locations.
In this study, we do not consider optimizing the GRIEF weights, $\mat{w}$.
VFE with $m=4$ inducing points achieves a root-mean squared error (RMSE) of 0.47 on the test set, whereas
GP-GRIEF with the same number of basis functions%
\footnote{Technically, VFE gives an infinite basis function expansion through a correction term, however, we will assume $p=m=4$.},
$p=4$, achieves an RMSE of 0.34, identical to the reconstruction provided by a full GP using the exact kernel.
While GP-GRIEF uses a dense grid of $m=25$ inducing points, 
it has a computational complexity equivalent to VFE.
This demonstrates the reconstruction power of GP-GRIEF, even when very few eigenfunctions are considered. \vspace{-3mm}
\begin{figure*}[t]
	\centering
	\begin{subfigure}[b]{0.33\textwidth}
		\centering
		\includegraphics[width=0.9\textwidth]{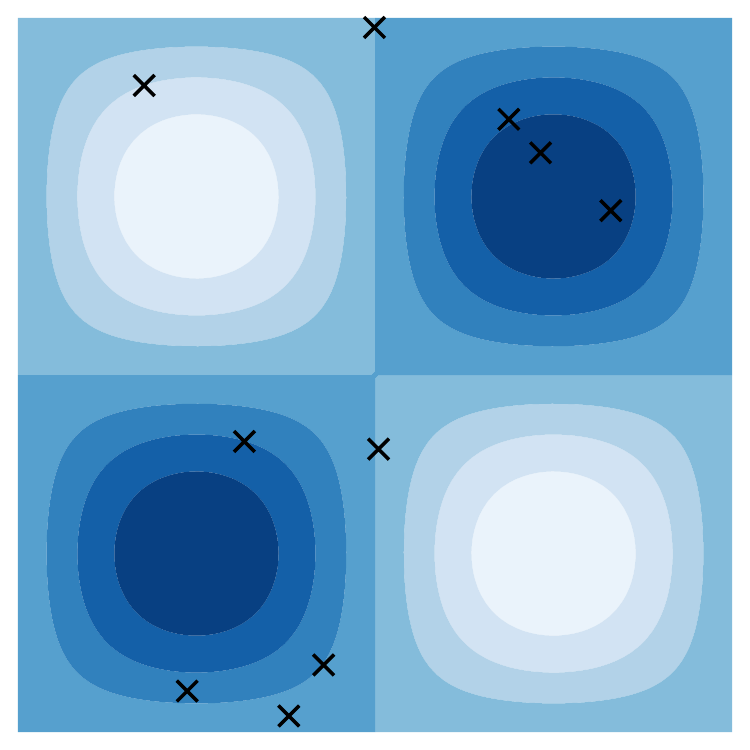}
		\caption{Test Data.}
		\label{fig:exact}
	\end{subfigure}%
	\begin{subfigure}[b]{0.33\textwidth}
		\centering
		\includegraphics[width=0.9\textwidth]{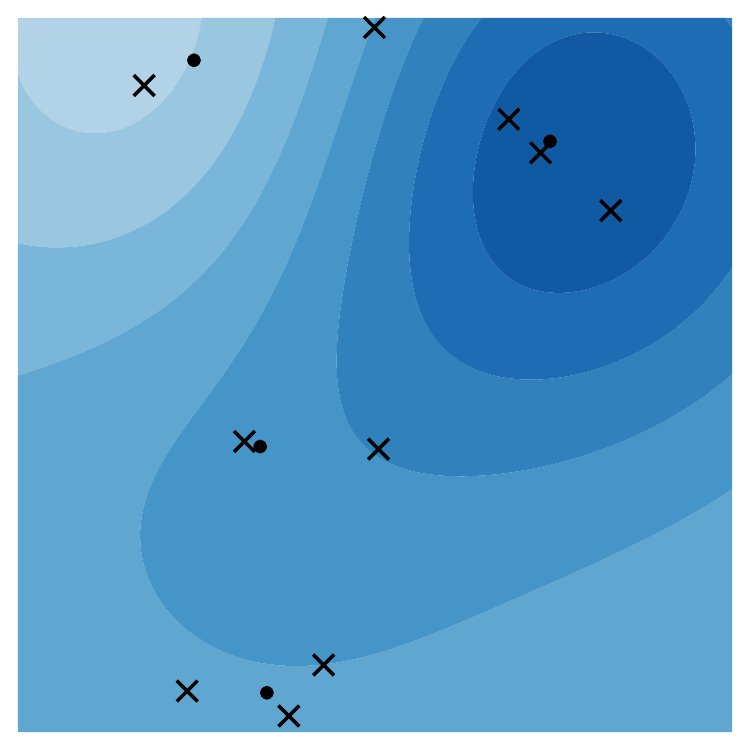}
		\caption{VFE, $m=4$, RMSE${=}0.47$}
		\label{fig:vi}	
	\end{subfigure}%
	\begin{subfigure}[b]{0.33\textwidth}
		\centering
		\includegraphics[width=0.9\textwidth]{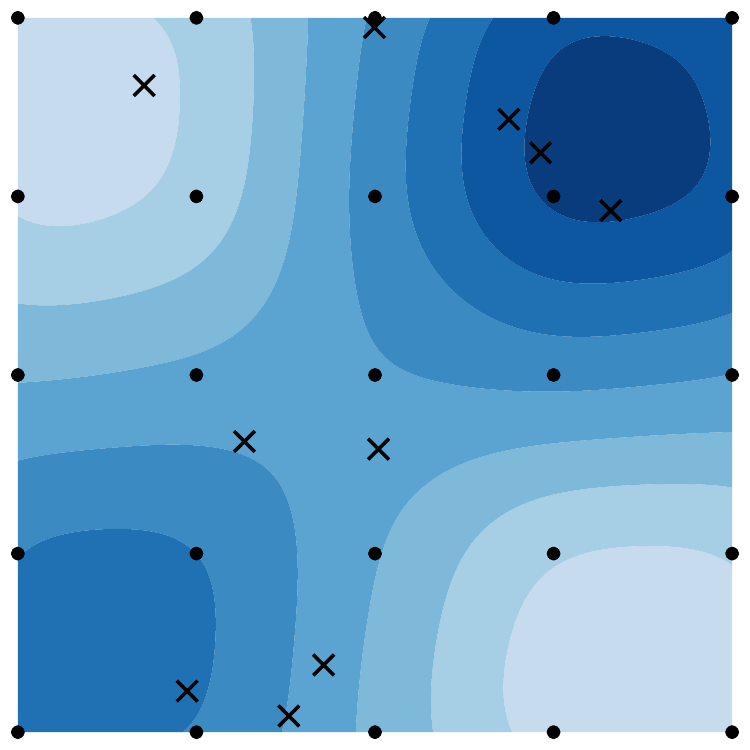}
		\caption{GP-GRIEF, $m{=}25,\ p{=}4$,~RMSE${=}0.34$}
		\label{fig:proposed}	
	\end{subfigure}
	\caption{Reconstruction using GP-GRIEF outperforms VFE.
		Both techniques use an equal number of basis functions and have the same computational complexity.
		Crosses denote training point positions and dots denote inducing point locations.
		\\ \vspace{-7mm} \mbox{} 
		}
	\label{fig:2d}
\end{figure*}

\paragraph{Kernel Reconstruction Accuracy}
We compare the kernel covariance reconstruction error of the GRIEF \nystrom\ method to competing techniques in \cref{fig:nystrom_reconstruction}.
We sample $5000$ points from $\mathcal{U}({-}\sqrt{3},\sqrt{3})$ in $d=100$ dimensions, randomly taking half for training and half for testing, and we consider a squared-exponential kernel.
Given only the training set, we attempt to reconstruct the exact prior covariance matrices between the training set (\cref{fig:train_cov}), and the joint train/test set~(\ref{fig:joint_cov}).
This allows us to study the kernel reconstruction accuracy between points both within, and beyond the training set.
In both studies, the proposed GRIEF \nystrom\ method greatly outperforms a 
random Fourier features reconstruction~\cite{rahimi_rff}, and a
randomized \nystrom\ reconstruction where $m=p$ inducing points are uniformly sampled from the training set. 
We emphasize that both randomized \nystrom\ and GRIEF \nystrom\ have the same computational complexity for a given number of basis functions,~$p$, even though the GRIEF \nystrom\ method uses $m = 10^{200}$ inducing points ($\widebar{m}=100$).

In the joint train/test study of \cref{fig:joint_cov}, we observe a larger gap between GRIEF \nystrom\ and randomized \nystrom\ than in \cref{fig:train_cov}.
This is not surprising since the goal of the randomized \nystrom\ method (and indeed all existing extensions to this technique) is to improve the accuracy of eigenfunctions evaluated on the training set which does not guarantee performance of the eigenfunctions evaluated on points outside this set.
For example, if a test point is placed far from the training set then we expect a poor approximation from existing \nystrom\ methods.
However, our GRIEF \nystrom\ approach attempts to fill out the input space with inducing points everywhere, not just near training points. 
This guarantees an accurate approximation at test locations even if they are far from training data.
Comparatively, the random Fourier features technique samples from a distribution that is independent of the training data, so it is also expected to perform no worse on the joint set than the training set. 
However, we observe that it provides an equally poor reconstruction on both sets.

The black curves in \cref{fig:nystrom_reconstruction} show the exact eigen-decomposition of the covariance matrices which demonstrates the optimal reconstruction accuracy for a kernel approximation with a given number of basis functions.
We observe that GP-GRIEF approaches this optimal accuracy in both studies, even though the test point distribution is not known at training time.
\begin{figure}[t]
	\begin{subfigure}[b]{\linewidth}
	\includegraphics[width=\linewidth]{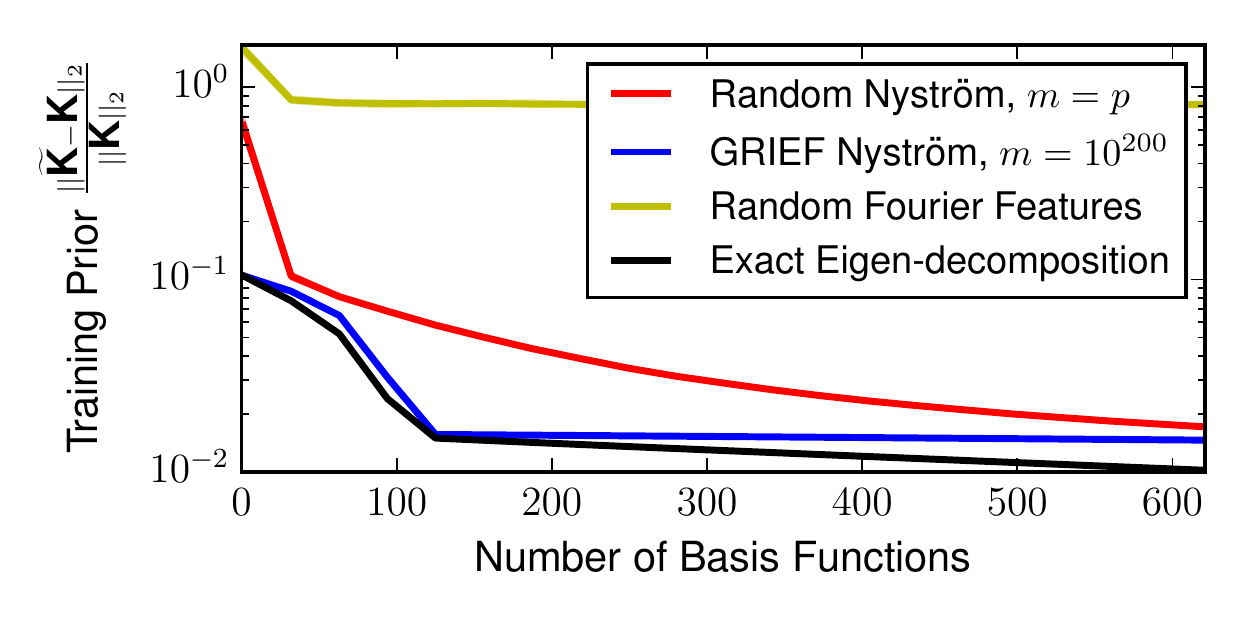}
	\vspace{-6mm}
	\caption{Training prior covariance error.}
	\label{fig:train_cov}
	\end{subfigure}
	\begin{subfigure}[b]{\linewidth}
	\includegraphics[width=\linewidth]{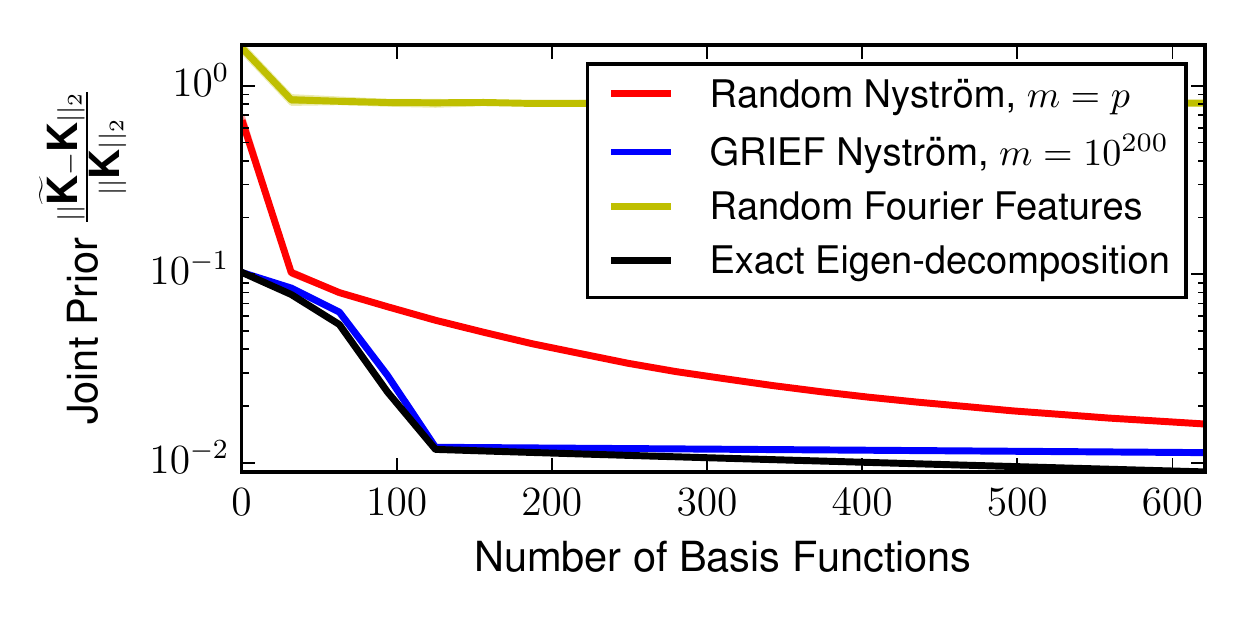}
	\vspace{-6mm}
	\caption{Train/test joint prior covariance error.}
	\label{fig:joint_cov}
	\end{subfigure}
	\caption{Covariance matrix reconstruction error of GP-GRIEF beats randomized \nystrom\ with uniform sampling averaged over 10 samples.
	GP-GRIEF approaches the optimal reconstruction accuracy of the black curve.
	\\ \vspace{-6mm} \mbox{} 
	}
	\label{fig:nystrom_reconstruction}
	
\end{figure}

\paragraph{UCI Regression Studies}
\begin{table*}[t]
	\centering
	\def\arraystretch{1.1}
	{\footnotesize
		\input{uci_results.csv}}
	\caption{
		Mean and standard deviation of test error and average training time (including hyperparameter estimation or MCMC sampling) from 10-fold cross validation (90\% train, 10\% test per fold) on UCI regression datasets. 
	}
	\label{tbl:uci}
\end{table*}
We next assess performance on real-world regression datasets from the UCI repository.
Using the authors' code%
\footnote{\url{https://github.com/treforevans/gp_grief}}, we report the mean and standard deviation of the RMSE from 10-fold cross validation%
\footnote{90\% train, 10\% test per fold. We use folds from\\\url{https://people.orie.cornell.edu/andrew/code}}.
Also presented is the mean training time per fold on a machine with two E5-2680 v3 processors.
We use a squared-exponential kernel with automatic relevance determination (SE-ARD) and we compare our test errors to those reported by \citet{yang_a_la_carte} using type-II inference on the same train-test splits.
\citet{yang_a_la_carte} used an exact GP with an SE-ARD kernel for datasets with $n < 2000$, and Fastfood expansions were used to approximate the SE-ARD kernel for the larger datasets ($n>2000$).

Before training the GP-GRIEF models, we initialize the base kernel hyperparameters, $\mat{\theta}$, by maximizing the marginal likelihood of an exact GP constructed on $\min(n,1000)$ points randomly selected from the dataset. 
We then train GP-GRIEF-II or GP-GRIEF-I which differ by not only type-II or type-I inference but also the kernel parametrizations used.
GP-GRIEF-II uses the kernel from \cref{eqn:nystrom_kernel} which is parametrized by the base kernel hyperparameters: $\{\mat{\theta},\sigma^2\}$.
The presented training time includes log marginal likelihood maximization to estimate the hyperparameters, beginning with the initialized values.
GP-GRIEF-I uses the kernel from \cref{eqn:reweighted_kernel} which is parametrized by the basis function weights:
$\{\mat{w},\sigma^2\}$; the base kernel hyperparameters, $\mat{\theta}$, are fixed to the initialized values.
We integrate out $\{\mat{w},\sigma^2\}$ using Metropolis adjusted Langevin dynamics MCMC which uses gradient information~\cite{girolami_MALA}.
The training time includes MCMC sampling, which we run for 10000 iterations.
We use log-normal priors with \{mode, variance\} of 
$\{1,100\}$ for $\mat{w}$, and 
$\{\sigma_0^2,0.04\}$ for $\sigma^2$, where $\sigma_0^2$ is the initialized value. 
We begin sampling at the prior mode, burning the first 1000 samples and thinning every 50 thereafter.
For datasets with $n > 10^6$ we use the \order{p} LML computations described in \cref{sec:mercer_expansion}.
For all studies, we fix $\widebar{m}=10$, 
and we fix $p=1000$ for GP-GRIEF-I.
For GP-GRIEF-II, we make $p$ proportional to $n$ by rounding $n$ down to the nearest power of ten, or take 1000 if it is lesser, i.e.~$p=\min(1000,10^{\lfloor \log_{10} n \rfloor})$.

It is firstly evident that both GP-GRIEF-I and GP-GRIEF-II outperform the exact GP presented by \citet{yang_a_la_carte} on nearly every small dataset ($n < 2000$).
In particular, GP-GRIEF-I performs extremely well on these small datasets as we would expect since it uses a very flexible kernel and is robust to over-fitting as a result of the principled Bayesian approach employed.
On larger datasets, where we expect the hyperparameter posterior to be more peaked, we see that the type-II techniques begin to be competitive.
On these larger datasets, both GP-GRIEF techniques show comparable test error to \citet{yang_a_la_carte} on all datasets but perform considerably better on kin40k and the electric dataset with two-million training points.
With respect to time, we note that the GP-GRIEF-I model trained extremely rapidly considering a fully Bayesian approach was taken;
only 25 minutes were required for the two-million point electric dataset even though this size is prohibitive for most GP models taking a type-II empirical Bayes approach.

The independence of computational complexity on $m$ allows enormous numbers of inducing points to be used.
We use $m{=}10^{33}$ inducing points for the cancer dataset which demonstrates the efficiency of the matrix algebra employed since storing a double-precision vector of this length requires 
8~quadrillion exabytes; 
far more than all combined hard-disk space on earth.

\section{Conclusion}
Our new technique, GP-GRIEF, has been outlined along with promising initial results on large real-world datasets where we demonstrated GP training and inference in \order{p} time with \order{p} storage.
This fast training enables type-I Bayesian inference to remain computationally attractive even for very large datasets as we had shown in our studies.
We showed that our complexities are independent of $m$, allowing us to break the \emph{curse of dimensionality} inherent to methods that manipulate distributions of points on a full Cartesian product grid.
Asymptotic results were also presented to show why a choice of large $m$ is important to provide an accurate global kernel approximation.
Lastly, we considered the use of up to $10^{33}$ inducing points in our regression studies, demonstrating the efficiency of the matrix algebra employed.
We discussed how the developed algebra can be used in areas beyond the focus of the numerical studies, such as in a general kernel interpolation framework, or in general kernel matrix preconditioning applications.
However, it will be interesting to explore what other applications could exploit the developed matrix algebra techniques.

\FloatBarrier
{
	\section*{Acknowledgements} Research funded by an NSERC Discovery Grant and the Canada Research Chairs~program.
}
\dobib
\FloatBarrier
\newpage \mbox{}
\newpage
\appendix

\section{Re-Weighted Eigenfunction Kernel Derivatives}
\label{sec:typeI_derivatives}
Calculating the $p+1$ derivatives of the log marginal likelihood (LML) with respect to $\{\mat{w},\sigma^2\}$ using finite difference approximations would require \order{p^4} time.
We show how all these derivatives can be analytically computed in \order{p^3} time. 
We also discuss how the use of transformed basis functions (i.e. replacing $\mPhi$ with $\widetilde{\mPhi}$) allows derivative computations in \order{p}.

The LML can be written as follows
\begin{multline*}
\log \mathcal{P}(\mat{y} | \mat{\theta}, \sigma^2, \mat{X}) = 
-\tfrac{n}{2} \log(2\pi) \\
-\tfrac{1}{2} \underbrace{\log | \Kxx + \sigma^2 \I{n} |}_\text{Complexity}
-\tfrac{1}{2} \underbrace{\mat{y}^T\alp}_\text{Data Fit},
\end{multline*}
where $\alp =  (\Kxx + \sigma^2 \I{n})^{-1} \mat{y}$.
For clarity, we will derive the gradients of the complexity and data-fit terms separately.

\paragraph{Data-Fit Weight Derivatives}
First we show how the derivative of the data-fit term can be computed within this time
\begin{align*}
\frac{\partial \mat{y}^T \mat{\alpha}}{\partial w_i} 
= -\mat{\alpha}^T \frac{\partial \widetilde{\mat{K}}}{\partial w_i} \mat{\alpha} 
= -\big(\mat{\phi}_i^T \mat{\alpha}\big)^2
\end{align*}
where $\mat{\alpha} =  \big(\widetilde{\mat{K}}_\text{X,X} + \sigma^2\mat{I}_n\big)^{-1} \mat{y} \inR{n}$, and
we make the observation that $\frac{\partial \widetilde{\mat{K}}}{\partial w_i} = \mat{\phi}_i\mat{\phi}_i^T$.
We can vectorize this to compute all data-fit derivatives
\begin{align*}
\frac{\partial \mat{y}^T \mat{\alpha}}{\partial \mat{w}}
&= -\big(\mat{\Phi}^T \mat{\alpha}\big)^2,\\
&= - \sigma^{-4} \big(\mPhi^T\y - \mat{A}\mat{P}^{-1}\mPhi^T\y\big)^2,
\end{align*}
where
$\mat{A} = \mat{\Phi}^T\mat{\Phi} \inR{p \times p}$ and
$\mPhi^T\y \inR{p}$ are both precomputed before LML iterations begin, and
$\mat{P} = \sigma^2 \mat{W}^{-1} + \mat{A} \inR{p \times p}$ is also required to compute the LML (see \cref{eqn:inv_det_lemmas}) so it is already computed and factorized.
Evidently, the data-fit term derivatives can be computed in \order{p^3} at each LML iteration.

\paragraph{Complexity Term Weight Derivatives}
Now we derive the complexity term gradient.
\begin{align*}
\frac{\partial \log|\widetilde{\mat{K}} + \sigma^2 \mat{I}_n|}{\partial w_i} 
&= \text{Tr}\bigg[\big(\widetilde{\mat{K}} + \sigma^2 \mat{I}_n \big)^{-1} \frac{\partial \widetilde{\mat{K}}}{\partial w_i} \bigg],
\end{align*}
using $\frac{\partial \widetilde{\mat{K}}}{\partial w_i} = \mat{\phi}_i\mat{\phi}_i^T$
and the cyclic permutation invariance of the trace operation, we get
\begin{align*}
\frac{\partial \log|\widetilde{\mat{K}} + \sigma^2 \mat{I}_n|}{\partial w_i} 
&= \mat{\phi}^T_i \big(\widetilde{\mat{K}} + \sigma^2 \mat{I}_n \big)^{-1} \mat{\phi}_i.
\end{align*}
Using the matrix inversion lemma, the preceding equation becomes
\begin{align*}
\frac{\partial \log|\widetilde{\mat{K}} + \sigma^2 \mat{I}_n|}{\partial w_i} 
&= \sigma^{-2} \big(\mat{\phi}^T_i \mat{\phi}_i - \mat{\phi}^T_i\mat{\Phi}\mat{P}^{-1}\mat{\Phi}^T \mat{\phi}_i\big),\\
&= \sigma^{-2} \big(a_{ii} - \mat{a}_i^T\mat{P}^{-1}\mat{a}_i\big).
\end{align*}
Evidently, the complexity term derivatives each require \order{p^2} time so all $p$ derivatives can be computed in \order{p^3}. 
We can write the vectorized computation as
\begin{align*}
\frac{\partial \log|\widetilde{\mat{K}} + \sigma^2 \mat{I}_n|}{\partial \mat{w}}
=  \sigma^{-2}\big[\text{diag}\big(\mat{A}\big) - \big(\mat{A} \odot \mat{P}^{-1}\mat{A}\big)^T \mat{1}_p\big],
\end{align*}
where it is evident that the dominating expense $\mat{P}^{-1}\mat{A}$ is also required for the data-fit derivatives.

\paragraph{Noise Variance Derivatives}
We show here how the derivatives of the LML with respect to $\sigma^2$ can be computed in \order{n+p}, as follows
\begin{multline*} 
\frac{\partial \mat{y}^T \mat{\alpha}}{\partial \sigma^2} 
= -\mat{\alpha}^T \frac{\partial (\sigma^2 \I{n})}{\partial \sigma^2} \mat{\alpha} 
=-\mat{\alpha}^T \mat{\alpha},\\
= -\sigma^{-4}\big(\y^T\y - 2\y^T\mPhi\mat{P}^{-1}\mPhi^T\y + \y^T\mPhi\mat{P}^{-1}\mat{A}\mat{P}^{-1}\mPhi^T\y\big),
\end{multline*}
and
\begin{align*}
\frac{\partial \log|\widetilde{\mat{K}} + \sigma^2 \mat{I}_n|}{\partial \sigma^2} 
&= \text{Tr}\bigg(\big(\widetilde{\mat{K}} + \sigma^2 \mat{I}_n \big)^{-1} \frac{\partial (\sigma^2 \I{n})}{\partial \sigma^2} \bigg)\\
&= \text{Tr}\Big( \sigma^{-2}\big[
\mat{I}_n - \mat{\Phi}\mat{P}^{-1}\mat{\Phi}^T
\big]\Big)\\
&= \sigma^{-2}\big[n - \text{Tr}\big(\mat{P}^{-1} \mat{A}\big)\big].
\end{align*}
Evidently the first relation can be computed in \order{p^3} if
$\y^T\y$, $\mPhi^T\y$, and $\mat{A}$ are precomputed, and 
the second relation can be computed in \order{p} since the matrix product $\mat{P}^{-1} \mat{A}$ has already been explicitly computed to compute the derivatives with respect to $\mat{w}$.

\paragraph{Final Expressions} 
Combining the derived expressions for the derivatives of the LML with respect to the $p+1$ hyperparameters $\{\mat{w}, \sigma^2\}$, it is evident that all computations can be performed in \order{p^3}.
We can write the final expressions as follows
\begin{align*}
\frac{\partial \text{LML}}{\partial \mat{w}}
&{=} \frac{\big(\mat{r} - \mat{A}\mat{P}^{-1}\mat{r}\big)^2}{2\sigma^{4}}
- \frac{\text{diag}\big(\mat{A}\big) - \big(\mat{A} \odot \mat{P}^{-1}\mat{A}\big)^T \mat{1}_p}{2\sigma^{2}},\\
\frac{\partial \text{LML}}{\partial \sigma^2} 
&{=} \frac{\y^T\y {-} 2\mat{r}^T\mat{P}^{-1}\mat{r} {+} \mat{r}^T\mat{P}^{-1}\mat{A}\mat{P}^{-1}\mat{r}}{2\sigma^{4}}
{-} \frac{n {-} \text{Tr}\big(\mat{P}^{-1} \mat{A}\big)}{2\sigma^{2}},
\end{align*}
where 
$\mat{r} = \mPhi^T\y \inR{p}$.
If we transform the basis functions by replacing $\mPhi$ with $\widetilde{\mPhi}$ then it is evident that the derivative computations can be made in \order{p} since both $\mat{A} = \widetilde{\mPhi}^T\widetilde{\mPhi} = \mat{I}_{\widetilde{p}}$ and $\mat{P} = \sigma^2 \mat{W}^{-1} + \mat{A}$ will be diagonal.

\section{Type-I Inference Extensions}
\label{sec:typeI_extensions}
Here we discuss extensions to the type-I inference approach described in \cref{sec:mercer_expansion} for the re-weighted eigenfunction kernel.
Namely, we consider more flexible kernel parameterizations that can be used while still admitting MCMC iterations with a complexity independent of the size of the training set, $n$.
To begin, consider the re-weighted kernel in \cref{eqn:reweighted_kernel} from a weight space perspective~\cite{rasmussen_gpml}.
In this case, we are constructing the generalized linear model
$\sum_{i=1}^p \alpha_i\phi_i(\mat{x})$, where $\mat{\alpha} \inR{p}$.
We assume the observed responses are corrupted by independent Gaussian noise with variance $\sigma^2$, and
we specify the prior $\mat{\alpha} \sim \mathcal{N}\big(\mat{0}, \mat{W}\big)$.
This model is identical to the GP specified in \cref{sec:mercer_expansion}, which was introduced from a function space perspective.
There we had shown that we can specify a hyper-prior on $\mat{w}$ and perform MCMC with a complexity independent of $n$.

We had previously taken $\mat{W}$ to be diagonal in \cref{sec:mercer_expansion} which assumes no prior correlation between the basis functions.
The technique can easily be extended by taking $\mat{W}$ to be dense and symmetric positive-definite.
\citet{pinheiro_cov_parameterization} discuss how such a matrix could be parameterized.
In this case, it is easily observed that computation of the log-marginal likelihood can still be computed with a complexity independent of $n$ from \cref{eqn:inv_det_lemmas}.
With a dense $\mat{W}$, the kernel can be written as
\begin{align} \label{eqn:dense_reweighted_kernel}
\widetilde{k}(\x,\z) 
&= \sum_{i=1}^p \sum_{j=1}^p 
w_{ij} \phi_i(\x) \phi_j(\z).
\end{align}

We may also consider specifying a non-zero prior mean for the basis function weights;
$\mat{\alpha} \sim \mathcal{N}\big( \mat{\mu}, \mat{W}\big)$, where $\mat{\mu} \inR{p}$.
This is equivalent to specifying a GP with the kernel in \cref{eqn:dense_reweighted_kernel} and the prior mean function $\sum_{i=1}^p \mu_i \phi_i(\mat{x})$.
In this case, it can be observed that computation of the log-marginal likelihood can still be performed with a complexity independent of $n$ through the relations in \cref{eqn:inv_det_lemmas} where
$\mat{r}$ is replaced by $\mat{r} - \mat{A}\mat{\mu}$,
and
$\mat{y}^T\mat{y}$ is replaced by $\mat{y}^T\mat{y} - 2\mat{r}^T\mat{\mu} + \mat{\mu}^T\mat{A}\mat{\mu}$. 
Similar to $\mat{W}$, we may also specify a hyper-prior on the elements of $\mat{\mu}$ and perform MCMC with these variables as well.
Both of these extensions increase the flexibility of the Bayesian model while ensuring that the computational complexity remains independent of the training dataset size.
This enables type-I inference to be performed on massive datasets.

\section{Source Code}
Source code that implements the methods discussed in the paper along with several tutorials can be found at\\ \url{https://github.com/treforevans/gp_grief}.
The code is implemented in Python and depends upon the \texttt{py-mcmc} package~\cite{bilionis_pymcmc}  for an implementation of Metropolis adjusted Langevin dynamics MCMC.
The code also depends on \texttt{GPy}~\cite{GPy} for its broad library of kernels.

\end{document}

%% file: uci_results.csv
\begin{tabular}{@{}llll|lp{6mm}l|lp{6mm}l|l|@{}}\toprule
Dataset&$n$&$d$&$m {=}\widebar{m}^d$&$p$&Time\newline\ \ (hrs)&GP-GRIEF-II&$p$&Time\newline\ \ (hrs)&GP-GRIEF-I&\citet{yang_a_la_carte}\\
\midrule
challenger&23&4&$10^{4}$&10&0&$0.554 \pm 0.277$&1000&0.178&$\bf 0.519 \pm 0.261$&$0.63\pm0.26$\\
fertility&100&9&$10^{9}$&100&0.001&$0.172 \pm 0.055$&1000&0.66&$\bf 0.166 \pm 0.051$&$0.21\pm0.05$\\
slump&103&7&$10^{7}$&100&0&$3.972 \pm 1.891$&1000&0.566&$\bf 3.470 \pm 1.712$&$4.72\pm2.42$\\
automobile&159&25&$10^{25}$&100&0.007&$0.145 \pm 0.057$&1000&0.604&$\bf 0.111 \pm 0.036$&$0.18\pm0.07$\\
servo&167&4&$10^{4}$&100&0&$0.280 \pm 0.085$&1000&0.265&$\bf 0.268 \pm 0.075$&$0.28\pm0.09$\\
cancer&194&33&$10^{33}$&100&0.007&$\bf 27.843 \pm 3.910$&1000&0.667&$30.568 \pm 3.340$&$35\pm4$\\
hardware&209&7&$10^{7}$&100&0&$0.408 \pm 0.046$&1000&0.637&$\bf 0.402 \pm 0.045$&$0.43\pm0.04$\\
yacht&308&6&$10^{6}$&100&0.001&$0.170 \pm 0.083$&1000&0.595&$\bf 0.120 \pm 0.070$&$0.16\pm0.11$\\
autompg&392&7&$10^{7}$&100&0.001&$2.607 \pm 0.356$&1000&0.594&$\bf 2.563 \pm 0.369$&$2.63\pm0.38$\\
housing&506&13&$10^{13}$&100&0.004&$3.212 \pm 0.864$&1000&0.62&$\bf 2.887 \pm 0.489$&$2.91\pm0.54$\\
forest&517&12&$10^{12}$&100&0.001&$1.386 \pm 0.14$&1000&0.621&$\bf 1.384 \pm 0.139$&$1.39\pm0.16$\\
stock&536&11&$10^{11}$&100&0.001&$\bf 0.005 \pm 0.000$&1000&0.567&$\bf 0.005 \pm 0.000$&$0.005\pm0.001$\\
energy&768&8&$10^{8}$&100&0.002&$0.49 \pm 0.057$&1000&0.622&$\bf 0.461 \pm 0.064$&$0.46\pm0.07$\\
concrete&1030&8&$10^{8}$&1000&0.008&$5.232 \pm 0.723$&1000&0.57&$5.156 \pm 0.766$&$\bf 4.95\pm0.77$\\
solar&1066&10&$10^{10}$&1000&0.003&$\bf 0.786 \pm 0.198$&1000&0.628&$0.809 \pm 0.193$&$0.83\pm0.20$\\
wine&1599&11&$10^{11}$&1000&0.012&$0.483 \pm 0.052$&1000&0.583&$0.477 \pm 0.047$&$\bf 0.47\pm0.08$\\
skillcraft&3338&19&$10^{19}$&1000&0.011&$\bf0.248 \pm 0.016$&1000&0.573&$\bf 0.248 \pm 0.016$&$0.25\pm0.02$\\
pumadyn&8192&32&$10^{32}$&1000&0.156&$\bf 0.20 \pm 0.00$&1000&0.645&$0.212 \pm 0.004$&$\bf 0.20\pm0.00$\\
elevators&16599&18&$10^{18}$&1000&0.283&$0.091 \pm 0.002$&1000&0.664&$0.097 \pm 0.001$&$\bf 0.090\pm0.001$\\
kin40k&40000&8&$10^{8}$&1000&0.38&$\bf 0.206 \pm 0.004$&1000&0.649&$\bf 0.206 \pm 0.004$&$0.28\pm0.01$\\
keggu&63608&27&$10^{27}$&1000&3.642&$\bf 0.118 \pm 0.003$&1000&0.704&$0.134 \pm 0.005$&$\bf 0.12\pm0.00$\\
3droad&434874&3&$10^{3}$&1000&0.869&$11.648 \pm 0.281$&1000&0.221&$12.966 \pm 0.077$&$\bf10.91\pm0.05$\\
electric&2049280&11&$10^{11}$&1000&8.019&$0.064 \pm 0.002$&1000&0.418&$\bf 0.058 \pm 0.006$&$0.12\pm0.12$\\
\bottomrule \end{tabular}